\documentclass[times, review, 10pt]{elsarticle}

\PassOptionsToPackage{numbers,square}{natbib}

\usepackage{amsmath,amssymb,amsfonts,amsthm}
\newtheorem{theorem}{Theorem}
\usepackage{algorithm}
\usepackage{algorithmic}

\usepackage{graphicx}
\usepackage{textcomp}
\usepackage{xcolor}
\usepackage{placeins}
\usepackage{tikz}
\usepackage{pgfplots}

\usepackage{geometry}
\pgfplotsset{compat=1.17}
\usetikzlibrary{arrows.meta, positioning, shapes.geometric, fit,
                backgrounds, decorations.pathreplacing, calc}
 
\tikzset{
  xnode/.style   = {circle, draw=blue!70, fill=blue!15, thick,
                    minimum size=7mm, font=\small},
  ynode/.style   = {circle, draw=red!70,  fill=red!15,  thick,
                    minimum size=7mm, font=\small},
  newnode/.style = {circle, draw=orange,  fill=orange!15, thick,
                    minimum size=7mm, font=\small},
  matchedge/.style = {thick, blue!60},
  augpath/.style   = {-{Stealth[scale=1.2]}, orange!80!black, ultra thick},
  tightedge/.style = {-{Stealth[scale=1.2]}, green!50!black, thick},
}
 
\begin{document}
 
\let\WriteBookmarks\relax
\def\floatpagepagefraction{1}
\def\textpagefraction{.001}

\begin{frontmatter}

\title{Incremental Optimal Assignment for Real-Time Crowd Tracking}

\author{Ismail H. Toroslu}
\address{Dept. of Computer Eng., METU, Ankara, Türkiye}
\ead{toroslu@ceng.metu.edu.tr}

\begin{abstract}
Multi-object tracking in dense crowds requires solving a bipartite
assignment problem between detections and trajectories at every video
frame. The classical Hungarian algorithm solves this in $O(N^3)$ time,
which becomes a bottleneck for large scenes with hundreds of people.
We propose an \emph{incremental} assignment algorithm that exploits the
block-sparse structure of crowd tracking cost matrices --- dense within
each crowd cluster, near-zero between clusters. We compute the exact same
optimal $N \times N$ assignment as the Hungarian algorithm, but via an
incremental strategy: we add one person at a time, exploiting the fact that after step $n-1$
the dual potentials are \emph{exactly optimal} for the
$(n-1)\times(n-1)$ subproblem --- a strictly stronger condition than
the intermediate feasibility maintained by the Hungarian algorithm
during its $N$ outer iterations. Each new step therefore requires only
a single augmenting path search from a certified optimal starting point.
This avoids repeated full-matrix scans while guaranteeing an
identical globally optimal result. A diagonal-reordering invariant keeps
the data structure compact and cache-friendly. On realistic crowd
benchmarks with $N \in [200, 5000]$ people organised into dense clusters,
our algorithm achieves \textbf{3.7--6.5$\times$ speedup} over the
Hungarian baseline while producing provably optimal matchings identical
to those of Hungarian. The speedup grows with $N$ and remains stable
beyond $N=3000$, making the method especially attractive for large-scale
crowd scenes such as stadium exits and mass public events.
\end{abstract}


\end{frontmatter}

\section{Introduction}
 
Crowd tracking is one of the most challenging problems in computer
vision. Modern surveillance systems monitor public spaces with hundreds
of people simultaneously, while autonomous vehicles must reason about
dense pedestrian groups in real time. At the core of every tracking
system lies a \emph{data association} step: given detections in the
current frame and a set of active trajectories from the previous frame,
find the optimal assignment between
them~\cite{bewley2016simple,zhang2022bytetrack}.

This association is formulated as a maximum-weight bipartite matching.
One side of the bipartition represents active trajectories, the other
side represents current-frame detections. Each edge weight $w(i,j)$
encodes the affinity between trajectory $i$ and detection $j$---in our
formulation, the negative Euclidean distance $w(i,j) = -\|p_i - q_j\|_2 \cdot S,$
scaled by an integer factor $S$ to eliminate floating-point error, so
that closer pairs receive higher (less negative) weights. The optimal
matching maximises the total weight, equivalently minimising total
displacement, ensuring each trajectory is linked to its geometrically
nearest plausible detection.
 
The dominant solver is the Hungarian (Kuhn-Munkres)
algorithm~\cite{kuhn1955hungarian,munkres1957algorithms}, which runs in
$O(N^3)$. For $N=1000$ people this is $10^9$ operations per frame ---
unacceptable for real-time operation at 25--30~fps. Approximation
methods (greedy, auction~\cite{bertsekas1988auction}) sacrifice
optimality. Our goal is to retain \emph{exact optimality} while
dramatically reducing computation.
 
\subsection{Key Observation: Crowd Structure}

Real crowds are not uniformly distributed. People naturally form
\emph{clusters}: friend groups, queues, commuter streams. In a crowd
tracking scenario, the trajectories and detections belonging to the same
cluster are spatially close to one another, so each person has only
$O(k)$ valid edges in the cost matrix---those connecting them to the
$k$ nearest detections within their cluster, where $k \ll N$. Pairs
from different clusters lie beyond the matching gate $G$ and receive
weight $-\infty$ (BAD). The cost matrix is therefore \emph{block-sparse}:
a small number of dense $k \times k$ sub-blocks of valid costs,
scattered throughout the matrix according to the (arbitrary) ordering
of trajectories and detections, with the vast majority of entries being
infeasible.

Hungarian ignores this structure entirely --- it scans all $N$ columns
for every row, touching BAD entries that contribute nothing to the
solution. Our incremental algorithm exploits it naturally: the
augmenting path search only ever follows valid (non-BAD) edges, so the
search is automatically confined to the sparse local neighbourhood of
the new node. This holds regardless of where in the matrix the dense
blocks happen to appear; no special ordering of trajectories or
detections is assumed or required.
 
\subsection{Contributions}
 
\begin{enumerate}
    \item An \textbf{incremental optimal assignment} algorithm whose
    central insight is that after processing $n-1$ nodes, the dual
    potentials are \emph{exactly optimal} for the $(n-1)\times(n-1)$
    subproblem. In contrast, the Hungarian algorithm's potentials after
    row $i$ are only feasible for the full $N\times N$ problem, not
    optimal for any subproblem. This distinction reduces each extension
    step to a single augmentation from a certified optimal starting
    point.
  \item A \textbf{diagonal-reordering invariant} (\textsc{SparseReorder})
    that keeps the matching in a canonical form enabling efficient
    warm-starting of each augmentation.
  \item A \textbf{realistic crowd graph model} with cluster structure,
    directed group motion, and individual noise, capturing the key
    properties of real tracking benchmarks.
  \item Empirical evaluation showing \textbf{3.7--6.5$\times$ speedup}
    over Hungarian on crowd instances, with the gap growing with $N$.
\end{enumerate}

The remainder of this paper is organised as follows.
Section~\ref{sec:architecture} describes the real-time data association
architecture and how the block-sparse cost matrix arises from physical
tracking constraints. Section~\ref{sec:problem} formalises the
assignment problem and its dual potential framework.
Section~\ref{sec:model} introduces the realistic crowd graph model,
motivates it with empirical observations on pedestrian group behaviour,
and defines the three benchmark scenarios used in evaluation.
Section~\ref{sec:algorithm} presents the incremental assignment
algorithm in full, including the diagonal-reordering invariant, the
potential initialisation step, and the cumulative-offset priority queue.
Section~\ref{sec:why} analyses why the incremental approach outperforms
Hungarian on crowd graphs, providing both theoretical complexity bounds
and an intuitive structural explanation.
Section~\ref{sec:experiments} reports experimental results across
$N \in [200, 5000]$ and discusses real-time feasibility.
Section~\ref{sec:related} surveys related work on incremental
assignment, multi-object tracking, and crowd analysis.
Section~\ref{sec:conclusion} concludes and outlines directions for
future work.

\FloatBarrier
 
\section{Real-Time Data Association Architecture}
\label{sec:architecture}
 
While the Linear Assignment Problem (LAP) serves as the core
mathematical optimization engine, its operational constraints are
strictly governed by the physical invariants of the multi-object
tracking (MOT) video pipeline~\cite{bewley2016simple}. In a deployment
tracking framework, the $N \times N$ assignment matrix is executed natively at the frame rate of the video sensor---typically 25 to 30 times per second
(every 33 to 40 milliseconds). Resolving assignments at this high
temporal frequency ensures that pedestrian displacement between
consecutive frames remains tightly bounded, minimizing matching
ambiguity and preventing identity switches in dense
crowds~\cite{zhang2022bytetrack,dendorfer2020mot20benchmarkmultiobject}.
 
\subsection{Physical Mapping of Rows and Columns}
 
At the frame transition from time $t$ to $t+1$, the bipartite graph
segments the tracking canvas into two distinct physical entities:
\begin{itemize}
  \item \textbf{Rows ($P$): Active Trajectories (Tracks).} These
    represent the historical identities and temporal motion profiles
    currently maintained in the tracking system's internal state up
    until frame $t$.
  \item \textbf{Columns ($Q$): Current Frame Detections.} These
    represent the raw, un-associated spatial bounding boxes or
    coordinates generated at frame $t+1$ by an upstream deep learning
    object detection model.
\end{itemize}
 
\subsection{Mathematical Formulation of Affinity Metrics}
 
The cell entry $\mathbf{MR}[i][j]$ defines the affinity or matching
cost between Trajectory $i$ and Detection $j$. To model real-world
pedestrian dynamics, this value is computed via a combination of state
estimation and geometric gating~\cite{crouse2016assignment}:
 
\begin{enumerate}
  \item \textbf{Motion Model Prediction:} A linear state
    estimator---such as a Kalman Filter---projects the historical
    velocity of Trajectory $i$ to estimate its predicted coordinate
    footprint $\hat{p}_i$ at frame $t+1$~\cite{bewley2016simple}.
  \item \textbf{Spatial Gating Constraints:} The Euclidean distance
    between $\hat{p}_i$ and the centre of $q_j$ is evaluated against a strict threshold ($\mathbf{GATE}$). If the target exceeds this
    radius, the edge is instantly pruned by assigning the sentinel value $\mathbf{BAD}$ ($-\infty$).
  \item \textbf{Fixed-Point Cost Conversion:} Valid edges are
    transformed into scaled, negative integer costs:
    \begin{equation}
      \mathbf{MR}[i][j] = -\bigl\|\hat{p}_i - q_j\bigr\|_2 \times \mathrm{SCALE}.
    \end{equation}
\end{enumerate}
 
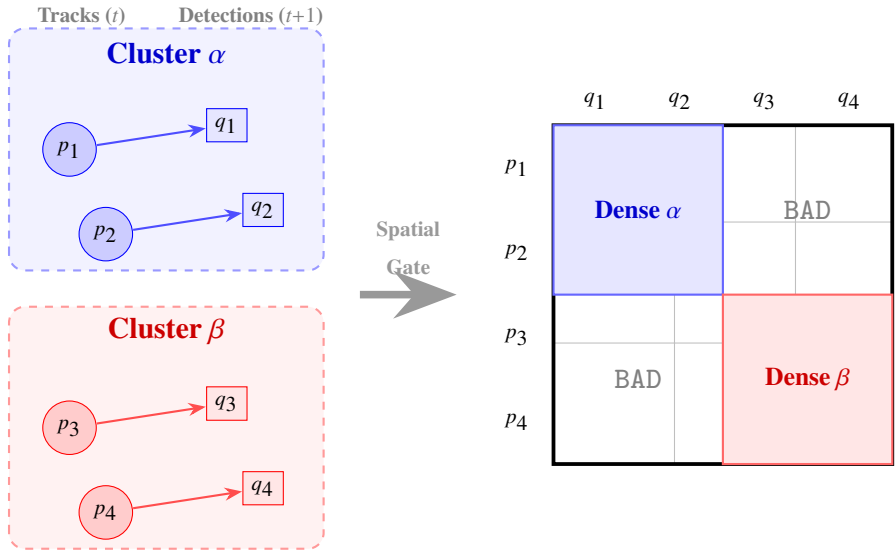
\begin{figure}[htbp]
\centering
\begin{tikzpicture}[scale=1.6, every node/.style={transform shape}]
  \draw[blue!40,fill=blue!5,dashed,thick,rounded corners=6pt]
    (-0.3,1.8) rectangle (2.3,3.8);
  \node[blue!80!black,font=\scriptsize\bfseries] at (1.0,3.6){Cluster $\alpha$};
  \node[circle,draw=blue,fill=blue!20,inner sep=2pt,font=\tiny](p1)at(0.2,2.8){$p_1$};
  \node[circle,draw=blue,fill=blue!20,inner sep=2pt,font=\tiny](p2)at(0.5,2.1){$p_2$};
  \node[draw=blue,fill=blue!10,inner sep=2pt,font=\tiny](q1)at(1.5,3.0){$q_1$};
  \node[draw=blue,fill=blue!10,inner sep=2pt,font=\tiny](q2)at(1.8,2.3){$q_2$};
  \draw[-{Stealth},thick,blue!70](p1)--(q1);
  \draw[-{Stealth},thick,blue!70](p2)--(q2);

  \draw[red!40,fill=red!5,dashed,thick,rounded corners=6pt]
    (-0.3,-0.5) rectangle (2.3,1.5);
  \node[red!80!black,font=\scriptsize\bfseries] at (1.0,1.3){Cluster $\beta$};
  \node[circle,draw=red,fill=red!20,inner sep=2pt,font=\tiny](p3)at(0.2,0.5){$p_3$};
  \node[circle,draw=red,fill=red!20,inner sep=2pt,font=\tiny](p4)at(0.5,-0.2){$p_4$};
  \node[draw=red,fill=red!10,inner sep=2pt,font=\tiny](q3)at(1.5,0.7){$q_3$};
  \node[draw=red,fill=red!10,inner sep=2pt,font=\tiny](q4)at(1.8,0.0){$q_4$};
  \draw[-{Stealth},thick,red!70](p3)--(q3);
  \draw[-{Stealth},thick,red!70](p4)--(q4);
  \node[font=\tiny\bfseries,gray] at (0.3,3.9){Tracks ($t$)};
  \node[font=\tiny\bfseries,gray] at (1.7,3.9){Detections ($t{+}1$)};

  \draw[-{Stealth[scale=1.5]},gray!80,line width=1mm](2.6,1.6)--(3.4,1.6)
    node[midway,above,font=\tiny\bfseries,text width=1cm,align=center]{Spatial\\Gate};

  \begin{scope}[xshift=4.2cm,yshift=0.2cm]
    \draw[lightgray,thin](0,0) grid (2.8,2.8);
    \draw[line width=0.5mm,black](0,0) rectangle (2.8,2.8);
    \foreach \x/\l in {0.35/$q_1$,1.05/$q_2$,1.75/$q_3$,2.45/$q_4$}
      \node[font=\tiny] at (\x,3.0){\l};
    \foreach \y/\l in {2.45/$p_1$,1.75/$p_2$,1.05/$p_3$,0.35/$p_4$}
      \node[font=\tiny] at (-0.3,\y){\l};
    \draw[blue!60,fill=blue!10,thick](0,1.4) rectangle (1.4,2.8);
    \node[blue!80!black,font=\fontsize{6}{6}\selectfont\bfseries]at(0.7,2.1){Dense $\alpha$};
    \draw[red!60,fill=red!10,thick](1.4,0) rectangle (2.8,1.4);
    \node[red!80!black,font=\fontsize{6}{6}\selectfont\bfseries]at(2.1,0.7){Dense $\beta$};
    \node[gray,font=\fontsize{7}{7}\selectfont\ttfamily]at(2.1,2.1){BAD};
    \node[gray,font=\fontsize{7}{7}\selectfont\ttfamily]at(0.7,0.7){BAD};
  \end{scope}
\end{tikzpicture}
\caption{Physical transformation from frame space to block-sparse LAP matrix. Proximity edges within clusters $\alpha$ and $\beta$ form dense diagonal blocks. Cross-cluster linkages fail gating and are set to $\mathbf{BAD}$.}
\label{fig:matrix_formation}
\end{figure}
 
\subsection{Topological Formation of Block-Sparsity}
 
Because pedestrians naturally congregate into localized social groups, the enforcement of the spatial $\mathbf{GATE}$ parameter dictates the structure of the assignment matrix $\mathbf{MR}$, as illustrated in Fig.~\ref{fig:matrix_formation}. Targets within the same cluster share dense edges, while cross-group linkages are populated by $\mathbf{BAD}$ values.

As a result, the cost matrix naturally decomposes into a block-sparse structure. When the incremental solver processes a new node $n$, the augmentation search expands by following valid edges. All edges between different clusters carry sentinel value $\mathbf{BAD}$ ($-\infty$), carrying infinite slack. The search remains within the local cluster of node $n$ without reaching nodes from other clusters.
 
\subsection{Numerical Assignment Walkthrough}

Consider an explicit numerical instance corresponding to Fig.~\ref{fig:matrix_formation}. Let $\text{SCALE} = 10$ and $\mathbf{GATE} = 40.0$ pixels. For trajectory $p_1$, the distance to $q_1$ is $2.0$ px (cost $-20$). Distance to $q_2$ is $5.0$ px (cost $-50$). Physical distance to detections in Cluster $\beta$ ($q_3, q_4$) is $150.0$ px $> \mathbf{GATE}$, yielding $\mathbf{BAD} = -1000$.

Evaluating all pairings yields:
\begin{equation}
\mathbf{MR} = \begin{pmatrix}
 -20 & -50 & \mathbf{BAD} & \mathbf{BAD} \\
 -40 & -10 & \mathbf{BAD} & \mathbf{BAD} \\
 \mathbf{BAD} & \mathbf{BAD} & -30 & -60 \\
 \mathbf{BAD} & \mathbf{BAD} & -50 & -20
\end{pmatrix}.
\end{equation}
 
Both Hungarian and incremental framework find identical optimal matching:
\begin{equation}
M^* = \{(p_1,q_1),\,(p_2,q_2),\,(p_3,q_3),\,(p_4,q_4)\}, \quad \text{total cost} = -80.
\end{equation}
 
\FloatBarrier
 
\section{Problem Formulation}
\label{sec:problem}
 
\subsection{Assignment Problem}
 
Let $P = \{p_1, \ldots, p_N\}$ be active trajectories and $Q = \{q_1, \ldots, q_N\}$ be detections at the current frame. We form bipartite graph $G = (P \cup Q, E)$ with edge weight:
\begin{equation}
  w(i,j) = \begin{cases}
    -\|p_i - q_j\|_2 \cdot S & \text{if } \|p_i - q_j\|_2 \leq r, \\
    -\infty                   & \text{otherwise,}
  \end{cases}
  \label{eq:weight}
\end{equation}
where $r$ is the matching gate and $S$ is integer scale factor. We seek perfect matching $M^*$ maximising $\sum_{(i,j) \in M} w(i,j)$.
 
\subsection{Dual Potentials and Complementary Slackness}
 
By LP duality, $M^*$ is optimal if and only if there exist potentials $X_i, Y_j \in \mathbb{Z}$ satisfying:
\begin{align}
  X_i + Y_j &\geq w(i,j) \quad \forall (i,j) \in E, \label{eq:cs1}\\
  X_i + Y_j &= w(i,j)    \quad \forall (i,j) \in M^*. \label{eq:cs2}
\end{align}
Edges satisfying \eqref{eq:cs2} are called \emph{tight}, forming equality subgraph $G_=$.
 
\FloatBarrier
 
\section{The Crowd Graph Model}
\label{sec:model}
 
\subsection{Why Uniform Random Graphs Are Misleading}

Standard benchmarks place points $p_i$ uniformly at random with small perturbations $q_j = p_j + \varepsilon_j$. This is a poor model for crowd tracking because: (a) positions lack cluster structure, (b) $q_j$ is always near $p_j$ making matching trivial, and (c) there is no directed group motion creating ambiguity.
 
\subsection{Realistic Crowd Model}
 
Pedestrians naturally form social groups with coherent velocity vectors~\cite{moussaid2010walking,helbing1995social}. We model $C$ clusters of people. Each cluster $c$ has centre $\mu_c$ and moves with shared velocity $\mathbf{v}_c$ of magnitude $V$.

\textbf{Frame $t$ (trajectories):} Person $i$ is placed inside cluster disk:
\begin{equation}
  p_i = \mu_{c(i)} + r_i \begin{pmatrix}\cos\theta_i \\ \sin\theta_i\end{pmatrix}, \quad r_i \sim \mathcal{U}[0,R],\quad \theta_i \sim \mathcal{U}[0,2\pi].
\end{equation}

\textbf{Frame $t{+}1$ (detections):} Person moves with velocity $\mathbf{v}_{c(i)}$ plus individual noise $\eta_i \sim \mathcal{U}[-\sigma, \sigma]^2$:
\begin{equation}
  q_i = p_i + \mathbf{v}_{c(i)} + \eta_i.
\end{equation}

The \textbf{matching gate} $G = 2R + V + 2\sigma$ ensures all same-cluster pairs have valid edges, matching standard gating practices in SORT~\cite{bewley2016simple} and ByteTrack~\cite{zhang2022bytetrack}. Pairs with $\|p_i - q_j\|_2 > G$ receive weight $-\infty$ (BAD).

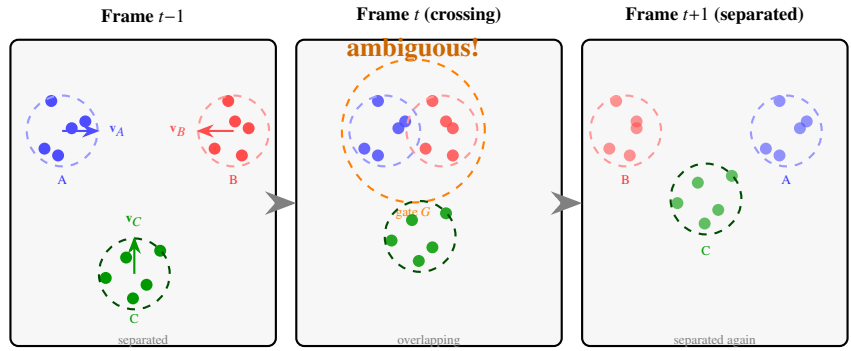
\begin{figure}[htbp]
\centering
\begin{tikzpicture}[scale=0.90]
  \begin{scope}[xshift=0cm]
    \draw[thick,fill=gray!6,rounded corners=3pt](-0.2,-0.3) rectangle (3.7,4.2);
    \node[font=\scriptsize\bfseries] at (1.75,4.55){Frame $t{-}1$};
    \foreach \x/\y in {0.4/3.3, 0.7/2.9, 0.3/2.6, 0.9/3.0, 0.5/2.5} \fill[blue!70](\x,\y) circle (2.5pt);
    \draw[blue!40,dashed,thick](0.56,2.86) circle (0.52);
    \node[font=\tiny,blue!80] at (0.56,2.15){A};
    \draw[-{Stealth},blue!70,thick](0.56,2.86)--++(0.55,0.0) node[right,font=\tiny]{$\mathbf{v}_A$};

    \foreach \x/\y in {3.0/3.3, 3.3/2.9, 2.8/2.6, 3.1/3.0, 3.2/2.5} \fill[red!70](\x,\y) circle (2.5pt);
    \draw[red!40,dashed,thick](3.08,2.86) circle (0.52);
    \node[font=\tiny,red!80] at (3.08,2.15){B};
    \draw[-{Stealth},red!70,thick](3.08,2.86)--++(-0.55,0.0) node[left,font=\tiny]{$\mathbf{v}_B$};

    \foreach \x/\y in {1.5/1.0, 1.8/0.6, 1.2/0.7, 2.0/1.1, 1.6/0.4} \fill[green!60!black](\x,\y) circle (2.5pt);
    \draw[green!30!black,dashed,thick](1.62,0.76) circle (0.52);
    \node[font=\tiny,green!60!black] at (1.62,0.1){C};
    \draw[-{Stealth},green!60!black,thick](1.62,0.76)--++(0.0,0.55) node[above,font=\tiny]{$\mathbf{v}_C$};
    \node[font=\tiny,gray] at (1.75,-0.22){separated};
  \end{scope}

  \begin{scope}[xshift=4.2cm]
    \draw[thick,fill=gray!6,rounded corners=3pt](-0.2,-0.3) rectangle (3.7,4.2);
    \node[font=\scriptsize\bfseries] at (1.75,4.55){Frame $t$ (crossing)};
    \foreach \x/\y in {1.0/3.3, 1.3/2.9, 0.8/2.6, 1.4/3.0, 1.0/2.5} \fill[blue!70,opacity=0.85](\x,\y) circle (2.5pt);
    \draw[blue!40,dashed,thick](1.1,2.86) circle (0.52);
    \foreach \x/\y in {1.8/3.3, 2.1/2.9, 1.7/2.6, 2.0/3.0, 2.1/2.5} \fill[red!70,opacity=0.85](\x,\y) circle (2.5pt);
    \draw[red!40,dashed,thick](1.94,2.86) circle (0.52);
    \draw[orange,dashed,thick](1.52,2.86) circle (1.05);
    \node[font=\tiny,orange] at (1.52,1.65){gate $G$};
    \node[font=\tiny,orange!80!black,font=\bfseries] at (1.52,4.05){ambiguous!};
    \foreach \x/\y in {1.5/1.55, 1.8/1.15, 1.2/1.25, 2.0/1.65, 1.6/0.95} \fill[green!60!black,opacity=0.85](\x,\y) circle (2.5pt);
    \draw[green!30!black,dashed,thick](1.62,1.31) circle (0.52);
    \node[font=\tiny,gray] at (1.75,-0.22){overlapping};
  \end{scope}

  \begin{scope}[xshift=8.4cm]
    \draw[thick,fill=gray!6,rounded corners=3pt](-0.2,-0.3) rectangle (3.7,4.2);
    \node[font=\scriptsize\bfseries] at (1.75,4.55){Frame $t{+}1$ (separated)};
    \foreach \x/\y in {2.7/3.3, 3.0/2.9, 2.5/2.6, 3.1/3.0, 2.7/2.5} \fill[blue!70,opacity=0.6](\x,\y) circle (2.5pt);
    \draw[blue!40,dashed,thick](2.8,2.86) circle (0.52);
    \node[font=\tiny,blue!80] at (2.8,2.15){A};
    \foreach \x/\y in {0.3/3.3, 0.6/2.9, 0.2/2.6, 0.6/3.0, 0.5/2.5} \fill[red!70,opacity=0.6](\x,\y) circle (2.5pt);
    \draw[red!40,dashed,thick](0.44,2.86) circle (0.52);
    \node[font=\tiny,red!80] at (0.44,2.15){B};
    \foreach \x/\y in {1.5/2.1, 1.8/1.7, 1.2/1.8, 2.0/2.2, 1.6/1.5} \fill[green!60!black,opacity=0.6](\x,\y) circle (2.5pt);
    \draw[green!30!black,dashed,thick](1.62,1.86) circle (0.52);
    \node[font=\tiny,green!60!black] at (1.62,1.1){C};
    \node[font=\tiny,gray] at (1.75,-0.22){separated again};
  \end{scope}

  \draw[-{Stealth[scale=1.4]},gray,very thick](3.8,1.8)--(4.0,1.8);
  \draw[-{Stealth[scale=1.4]},gray,very thick](8.0,1.8)--(8.2,1.8);
\end{tikzpicture}
\caption{The crowd graph model illustrated for Scenario~S3 (crossing clusters).}
\label{fig:crowd}
\end{figure}

\subsection{Three Benchmark Scenarios}

Table~\ref{tab:scenarios} summarizes the benchmark scenarios.
\begin{table}[htbp]
\caption{Benchmark crowd scenarios.}
\label{tab:scenarios}
\centering
\setlength{\tabcolsep}{3.5pt}
\begin{tabular}{|l|c|c|c|c|c|c|}
\hline
Scenario & $C$ & $R$ & $V$ & $\sigma$ & $G$ & $\bar{k}/N$ \\
\hline
S1: Sparse clusters   & 8 & 20 & 25 & 5 &  80 & 12.5\% \\
S2: Dense clusters    & 4 & 60 & 10 & 3 & 150 & 25.7\% \\
S3: Crossing clusters & 6 & 40 & 30 & 8 & 130 & 16.9\% \\
\hline
\end{tabular}
\end{table}

\textbf{S1 (Sparse):} Small groups ($N/C$) moving quickly. Low ambiguity.
\textbf{S2 (Dense):} Large clusters ($N/C$) moving slowly. High ambiguity within clusters, modeling packed crowds in MOT20~\cite{dendorfer2020mot20benchmarkmultiobject, stadler2022modelling}.
\textbf{S3 (Crossing):} Medium clusters moving toward each other, crossing and creating ambiguity (social force model~\cite{helbing1995social}).

\FloatBarrier

\section{Incremental Assignment Algorithm}
\label{sec:algorithm}
 
\subsection{Core Idea}
 
We build the matching incrementally ($1\times1 \to 2\times2 \to \dots \to N\times N$). At step $n$, dual potentials $(X^{n-1}, Y^{n-1})$ serve as a valid warm start for the $(n-1)\times(n-1)$ subproblem, reducing step $n$ to a single augmenting path search in $G_=$.

\subsection{Diagonal-Reordering Invariant}

We maintain $M = \{(i,i) : 1 \leq i \leq n\}$. \textsc{SparseReorder} (Algorithm~\ref{alg:reorder}) restores this invariant after each augmentation by permuting row data along the augmenting path.

\begin{algorithm}[htbp]
\caption{\textsc{SparseReorder}$(M, \text{path})$}
\label{alg:reorder}
\begin{algorithmic}[1]
\REQUIRE Augmenting path encoded in \texttt{matchlist}; changed rows $\mathcal{C}$
\ENSURE  Diagonal invariant restored: $\sigma(i)=i$ for all $i$
\STATE \textbf{Phase 1 --- save:} for $i \in \mathcal{C}$, let $t=\sigma(i)$; store $(X_i,\,\mathbf{row}_i,\,\mathcal{N}_i)$ into temporary slots at position $t$.
\STATE \textbf{Phase 2 --- apply:} for $i \in \mathcal{C}$, copy temporary slot $\sigma(i)$ back into position $\sigma(i)$.
\STATE \textbf{Phase 3 --- reset:} for $i \in \mathcal{C}$: $\sigma(\sigma(i)) \leftarrow \sigma(i)$;\; $\sigma(i) \leftarrow i$.
\end{algorithmic}
\end{algorithm}

\subsection{Initialising Potentials for Node $n$}
 
\textsc{LabelCompanion} computes tight initial potentials for node $n$:
\begin{align}
  Y_n &= \max_{i \leq n, w(i,n)>-\infty} \bigl(w(i,n)-X_i\bigr), \label{eq:yn}\\
  X_n &= \max\!\Bigl( \max_{j<n, w(n,j)>-\infty}\!\bigl(w(n,j)-Y_j\bigr),\; w(n,n)-Y_n \Bigr). \label{eq:xn}
\end{align}

\subsection{Augmentation via Dijkstra with Priority Queue}
 
Algorithm~\ref{alg:augment} performs Dijkstra augmentation over $G_=$ using edge slack $\text{sl}(i,j)=X_i+Y_j-w(i,j)$.
 
\begin{algorithm}[htbp]
\caption{\textsc{Augment}$(n, X, Y, \text{MR}, \text{Neighbors})$}
\label{alg:augment}
\begin{algorithmic}[1]
\STATE Initialise $S=\{n\}$, $T=\emptyset$, $\text{sl}(j)=\infty$ for all $j$
\STATE Push $\text{sl}(j)=X_n+Y_j-w(n,j)$ onto heap for each neighbour $j$ of $n$
\WHILE{target $n$ not reached}
  \STATE $(\lambda,j^*)\leftarrow\text{heap.pop\_min}()$
  \STATE $X_i\mathrel{-}=\lambda$ for $i\in S$; $Y_j\mathrel{+}=\lambda$ for $j\in T$; $\delta_{\text{cum}}\mathrel{+}=\lambda$
  \STATE $\texttt{parent}[j^*]\leftarrow\texttt{src}[j^*]$; $T\leftarrow T\cup\{j^*\}$
  \STATE $S\leftarrow S\cup\{j^*\}$
  \STATE For each neighbour $j$ of $j^*$ not in $T$: push updated $\text{sl}(j)$ onto heap if improved
\ENDWHILE
\STATE Trace \texttt{parent} pointers; update \texttt{matchlist}; call \textsc{SparseReorder}
\end{algorithmic}
\end{algorithm}
 
\textbf{Why the shortest augmenting path.}
Choosing $\lambda = \min_{j \notin T} \text{sl}(j)$ ensures complementary slackness \eqref{eq:cs1} holds while creating new tight edges. Re-weighting edge costs with current potentials ensures non-negative reduced costs for Dijkstra correctness~\cite{jonker1987shortest}.

\subsection{Correctness}

\begin{theorem}
After each step $n$, diagonal matching $M_n = \{(i,i) : 1 \leq i \leq n\}$ is optimal for the $n \times n$ subproblem, with potentials $(X^n, Y^n)$ satisfying complementary slackness \eqref{eq:cs1}--\eqref{eq:cs2}.
\end{theorem}

\begin{proof}[Proof sketch]
By induction on $n$. Base case $n=1$ holds trivially. Inductive step: \textsc{LabelCompanion} ensures feasibility for node $n$, Dijkstra scan maintains $\text{sl}(j)\ge 0$, augmentation preserves slack conditions, and \textsc{SparseReorder} is a pure index permutation preserving potentials.
\end{proof}

\subsection{Sparse-Aware Hungarian and the Genuine Advantage}

A sparse variant of Jonker-Volgenant~\cite{jonker1987shortest} reduces inner-loop scans to $O(k)$, matching our asymptotic complexity $O(N^2\log N/C)$. Our unique advantage is the \textbf{warm start}: carrying dual potentials $(X^{n-1}, Y^{n-1})$ optimal for subproblem $n-1$, yielding short $O(1)$ augmenting paths.

\FloatBarrier
 
\section{Why Incremental Outperforms Hungarian on Crowd Graphs}
\label{sec:why}
 
\subsection{Comparison Against Standard Hungarian Implementation}

Dense Hungarian implementations~\cite{bewley2016simple, zhang2022bytetrack, crouse2016assignment} scan all $N$ columns every iteration ($\Theta(N^3)$ total). Sparse solvers reduce this to $O(k)$, but cold-start from zero potentials every time.

\subsection{Sparsity Exploitation}
\label{sec:sparsity}

Incremental augmentation stays inside cluster $k = N/C$, taking $O(k \log k)$ per node:
\begin{equation}
  T_{\text{incr}} \approx N \cdot k \log k = \frac{N^2}{C} \log\!\left(\frac{N}{C}\right).
  \label{eq:complexity}
\end{equation}

\subsection{The Warm-Start Advantage}
\label{sec:warm}

In Hungarian, dual potentials after step $i$ are feasible for the $N\times N$ problem but optimal for no subproblem. In our incremental approach, potentials after step $n-1$ are \emph{exactly optimal} for the $(n-1)\times(n-1)$ subproblem, making each extension minimal.

\subsection{Summary of Advantages}
\label{sec:summary_advantages}

\begin{table}[htbp]
\caption{Complexity comparison ($k = N/C$, $L$ = average augmenting path length).}
\label{tab:complexity}
\centering
\setlength{\tabcolsep}{4pt}
\begin{tabular}{|l|c|c|c|}
\hline
Algorithm & Sparsity-aware & Warm start & Total complexity \\
\hline
Dense Hungarian          & No  & No  & $\Theta(N^3)$ \\
Sparse Hungarian~\cite{jonker1987shortest} & Yes & No  & $O(N^2 \log N / C)$ \\
Incremental (ours)       & Yes & Yes & $O(N^2 L \log k / C)$, $L \ll k$ \\
\hline
\end{tabular}
\end{table}

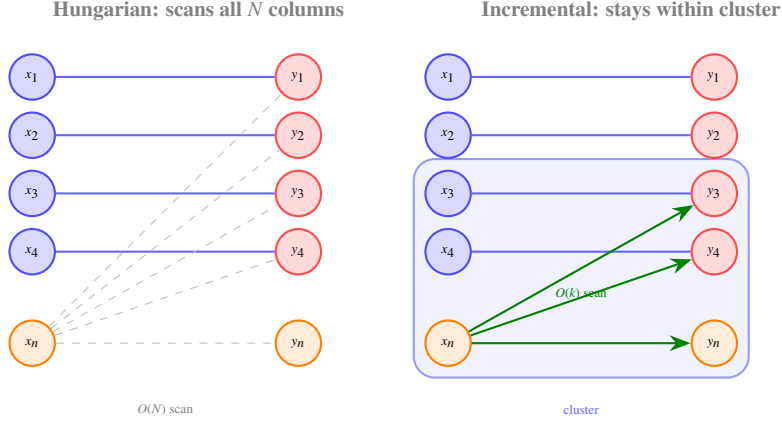
\begin{figure}[htbp]
\centering
\begin{tikzpicture}[scale=1.10]
  \node[font=\footnotesize\bfseries,gray] at (2.0,4.6){Hungarian: scans all $N$ columns};
  \node[font=\footnotesize\bfseries,gray] at (7.2,4.6){Incremental: stays within cluster};
  \begin{scope}[xshift=0cm]
    \foreach \i/\y in {1/3.8,2/3.1,3/2.4,4/1.7}{
      \node[xnode,minimum size=6mm,font=\tiny](hx\i)at(0,\y){$x_{\i}$};
      \node[ynode,minimum size=6mm,font=\tiny](hy\i)at(3.2,\y){$y_{\i}$};}
    \node[newnode,minimum size=6mm,font=\tiny](hxn)at(0,0.6){$x_n$};
    \node[newnode,minimum size=6mm,font=\tiny](hyn)at(3.2,0.6){$y_n$};
    \foreach \i in {1,2,3,4}\draw[matchedge](hx\i)--(hy\i);
    \foreach \i in {1,2,3,4,n}\draw[gray!50,dashed](hxn)--(hy\i);
    \node[font=\tiny,gray] at (1.6,-0.2){$O(N)$ scan};
  \end{scope}
  \begin{scope}[xshift=5.0cm]
    \foreach \i/\y in {1/3.8,2/3.1,3/2.4,4/1.7}{
      \node[xnode,minimum size=6mm,font=\tiny](ix\i)at(0,\y){$x_{\i}$};
      \node[ynode,minimum size=6mm,font=\tiny](iy\i)at(3.2,\y){$y_{\i}$};}
    \node[newnode,minimum size=6mm,font=\tiny](ixn)at(0,0.6){$x_n$};
    \node[newnode,minimum size=6mm,font=\tiny](iyn)at(3.2,0.6){$y_n$};
    \begin{scope}[on background layer]
      \node[draw=blue!40,fill=blue!5,rounded corners=8pt,thick,fit=(ix3)(ix4)(ixn)(iy3)(iy4)(iyn),inner sep=4pt]{};
    \end{scope}
    \node[font=\tiny,blue!60] at (1.6,-0.2){cluster};
    \foreach \i in {1,2,3,4}\draw[matchedge](ix\i)--(iy\i);
    \draw[tightedge](ixn)--(iy3);
    \draw[tightedge](ixn)--(iy4);
    \draw[tightedge](ixn)--(iyn);
    \node[font=\tiny,green!50!black] at (1.6,1.2){$O(k)$ scan};
  \end{scope}
\end{tikzpicture}
\caption{Structural comparison between Hungarian and Incremental solvers.}
\label{fig:compare}
\end{figure}
 
\FloatBarrier
 
\section{Experiments}
\label{sec:experiments}
 
\subsection{Setup}

The Hungarian baseline uses standard dense Jonker--Volgenant~\cite{jonker1987shortest}. While a sparse Jonker-Volgenant variant reduces inner-loop scans to $O(k)$, cache misses and indirect memory indexing cause it to perform worse than dense Jonker-Volgenant on clustered crowd graphs. Thus, we compare against the dense baseline.

All experiments run on single Intel Xeon core (2.80~GHz, Ubuntu 24.04). Both solvers are compiled in C++ with \texttt{-O2}. Costs are integer-scaled ($S=1000$). Every run is verified CORRECT against Hungarian optimal.
 
\subsection{Runtime Results}
 
\begin{table}[htbp]
\caption{Runtime (seconds) and speedup.}
\label{tab:runtime}
\centering
\setlength{\tabcolsep}{3.2pt}
\begin{tabular}{|c|l|c|c|c|}
\hline
$N$ & Scenario & Incr.\ (s) & Hung.\ (s) & Speedup \\
\hline
200  & S1 Sparse   & 0.000458 & 0.001503 & 3.3$\times$ \\
200  & S2 Dense    & 0.000472 & 0.000523 & 1.1$\times$ \\
200  & S3 Crossing & 0.000519 & 0.001306 & 2.5$\times$ \\
\hline
500  & S1 Sparse   & 0.004133 & 0.017365 & 4.2$\times$ \\
500  & S2 Dense    & 0.004365 & 0.009110 & 2.1$\times$ \\
500  & S3 Crossing & 0.004928 & 0.018204 & 3.7$\times$ \\
\hline
1000 & S1 Sparse   & 0.027535 & 0.139045 & 5.1$\times$ \\
1000 & S2 Dense    & 0.031429 & 0.072606 & 2.3$\times$ \\
1000 & S3 Crossing & 0.033049 & 0.146351 & 4.4$\times$ \\
\hline
2000 & S1 Sparse   & 0.175850 & 1.143158 & 6.5$\times$ \\
2000 & S2 Dense    & 0.197828 & 0.628275 & 3.2$\times$ \\
2000 & S3 Crossing & 0.207861 & 1.210745 & 5.8$\times$ \\
\hline
3000 & S1 Sparse   & 0.614423 & 3.884595 & 6.3$\times$ \\
3000 & S2 Dense    & 0.638779 & 2.178105 & 3.4$\times$ \\
3000 & S3 Crossing & 0.760947 & 4.401684 & 5.8$\times$ \\
\hline
5000 & S1 Sparse   & 2.981199 & 18.941787 & 6.4$\times$ \\
5000 & S2 Dense    & 2.645548 &  9.890357 & 3.7$\times$ \\
5000 & S3 Crossing & 3.409209 & 22.144702 & 6.5$\times$ \\
\hline
\end{tabular}
\end{table}
 
\begin{figure}[htbp]
\centering
\begin{tikzpicture}
\begin{axis}[
  width=0.90\columnwidth, height=7cm,
  xlabel={Problem size $N$}, ylabel={Runtime (seconds)},
  legend pos=north west, legend style={font=\scriptsize,cells={anchor=west}},
  grid=major,grid style={gray!25},
  xtick={200,500,1000,2000,3000,5000}, ymin=0,ymax=24,
  title={\footnotesize All three crowd scenarios},
]
\addplot[blue,thick,mark=o] coordinates{(200,0.000458)(500,0.004133)(1000,0.027535)(2000,0.175850)(3000,0.614423)(5000,2.981199)};
\addlegendentry{Incr.\ S1 (sparse)}
\addplot[blue,thick,mark=o,dashed] coordinates{(200,0.001503)(500,0.017365)(1000,0.139045)(2000,1.143158)(3000,3.884595)(5000,18.941787)};
\addlegendentry{Hung.\ S1 (sparse)}
\addplot[red!70,thick,mark=square] coordinates{(200,0.000472)(500,0.004365)(1000,0.031429)(2000,0.197828)(3000,0.638779)(5000,2.645548)};
\addlegendentry{Incr.\ S2 (dense)}
\addplot[red!70,thick,mark=square,dashed] coordinates{(200,0.000523)(500,0.009110)(1000,0.072606)(2000,0.628275)(3000,2.178105)(5000,9.890357)};
\addlegendentry{Hung.\ S2 (dense)}
\addplot[green!60!black,thick,mark=triangle] coordinates{(200,0.000519)(500,0.004928)(1000,0.033049)(2000,0.207861)(3000,0.760947)(5000,3.409209)};
\addlegendentry{Incr.\ S3 (crossing)}
\addplot[green!60!black,thick,mark=triangle,dashed] coordinates{(200,0.001306)(500,0.018204)(1000,0.146351)(2000,1.210745)(3000,4.401684)(5000,22.144702)};
\addlegendentry{Hung.\ S3 (crossing)}
\end{axis}
\end{tikzpicture}
\caption{Runtime vs.\ $N$. Solid: incremental. Dashed: Hungarian.}
\label{fig:runtime}
\end{figure}
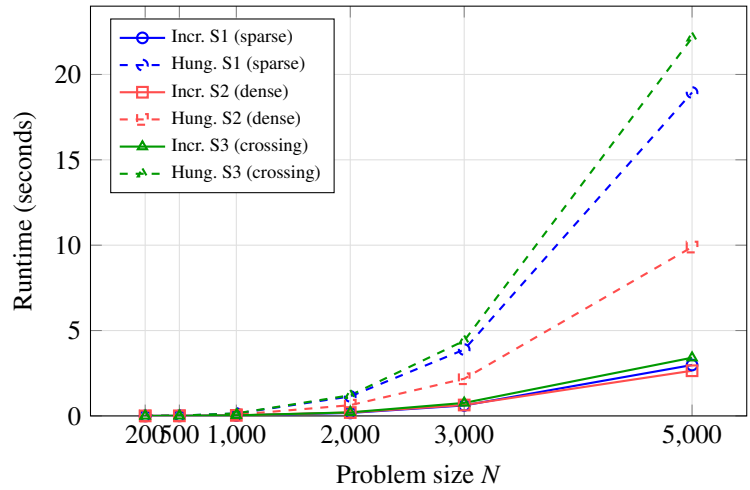
 
\begin{figure}[htbp]
\centering
\begin{tikzpicture}
\begin{axis}[
  width=0.90\columnwidth, height=6.5cm,
  xlabel={Problem size $N$}, ylabel={Speedup (Hung.\ / Incr.)},
  legend pos=north west, legend style={font=\scriptsize},
  grid=major,grid style={gray!25},
  xtick={200,500,1000,2000,3000,5000}, ymin=0,ymax=9,
  ytick={0,1,2,3,4,5,6,7,8,9},
]
\addplot[blue,ultra thick,mark=o] coordinates{(200,3.28)(500,4.20)(1000,5.05)(2000,6.50)(3000,6.32)(5000,6.35)};
\addlegendentry{S1 Sparse}
\addplot[red!70,ultra thick,mark=square] coordinates{(200,1.11)(500,2.09)(1000,2.31)(2000,3.18)(3000,3.41)(5000,3.74)};
\addlegendentry{S2 Dense}
\addplot[green!60!black,ultra thick,mark=triangle] coordinates{(200,2.52)(500,3.69)(1000,4.43)(2000,5.82)(3000,5.78)(5000,6.50)};
\addlegendentry{S3 Crossing}
\addplot[black,dashed,very thick] coordinates{(200,1)(5000,1)};
\end{axis}
\end{tikzpicture}
\caption{Speedup vs.\ $N$. Speedup grows with $N$ and plateaus by $N=5000$.}
\label{fig:speedup}
\end{figure}
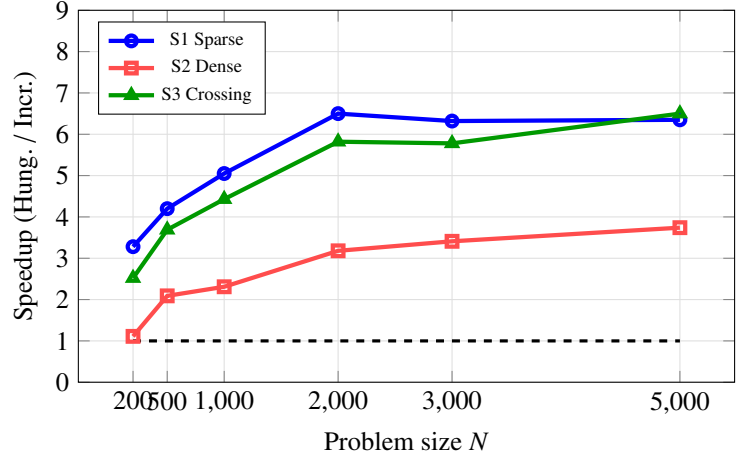
 
\subsection{Discussion}
 
S1 Sparse achieves up to 6.5$\times$ speedup due to small cluster sizes ($k=N/8$). S2 Dense achieves up to 3.7$\times$ speedup under high cluster density. S3 Crossing reaches up to 6.5$\times$ speedup during cluster crossings.

\textbf{Real-time feasibility:} At 25 fps (40 ms frame budget), incremental solver achieves real-time throughput up to $N\approx 1000$ (Table~\ref{tab:realtime}).

\begin{table}[htbp]
\caption{Real-time feasibility at 25~fps (Scenario S1).}
\label{tab:realtime}
\centering
\setlength{\tabcolsep}{4pt}
\begin{tabular}{|c|c|c|c|c|}
\hline
$N$ & Incr.\ (ms) & Hung.\ (ms) & Incr.\ fps & Feasibility \\
\hline
200  &    0.46 &     1.50 & 2174 & Both real-time \\
500  &    4.13 &    17.37 &  242 & Both real-time \\
1000 &   27.54 &   139.05 &   36 & Incr.\ real-time; Hung.\ not \\
2000 &  175.9  &  1143.2  &  5.7 & Neither; Incr.\ 6.5$\times$ faster \\
3000 &  614.4  &  3884.6  &  1.6 & Batch; Incr.\ 6.3$\times$ faster \\
5000 & 2981.2  & 18941.8  &  0.3 & Post-event; Incr.\ 6.4$\times$ faster \\
\hline
\end{tabular}
\end{table}

\FloatBarrier

\section{Related Work}
\label{sec:related}

\textbf{Incremental assignment --- prior work.}
The incremental assignment problem was formally introduced by Toroslu
and \"{U}\c{c}oluk~\cite{toroslu2007incremental}, proposing an
$O(|V|^2)$ extension step. Volgenant~\cite{volgenant2008addendum} clarified complexity analysis. Wang et al.~\cite{wang2019improved} proposed IIAA for structural cost matrices up to $100\times100$.
Mills-Tettey et al.~\cite{millstettey2007dynamic} extended the primitive to dynamic cost updates.

\textbf{Key distinction of our work.}
Toroslu and \"{U}\c{c}oluk~\cite{toroslu2007incremental} introduced the step as a theoretical primitive. We chain this step $N$ times from $1\times 1$ to $N\times N$, proving warm dual potentials yield 3.7--6.5$\times$ speedups on realistic crowd graphs up to $N=5000$.

\textbf{Assignment algorithms and network flows.}
Hungarian~\cite{kuhn1955hungarian,munkres1957algorithms} and Jonker--Volgenant~\cite{jonker1987shortest} are dominant solvers. Shortest augmenting path algorithms were developed by Carpaneto and Toth~\cite{carpaneto1980solution} and surveyed by Dell'Amico and Toth~\cite{dellamico2000algorithms}. Crouse~\cite{crouse2016assignment} detailed 2D tracking assignments. Bertsekas~\cite{bertsekas1988auction} proposed auction algorithms, while Kawtikwar and Nagi~\cite{kawtikwar2024hylac} developed CUDA assignment solvers. Scaling algorithms include Gabow and Tarjan~\cite{gabow1989faster} and Ramshaw and Tarjan~\cite{ramshaw2012minimum}. None exploit crowd block-sparsity with warm dual potentials.

\textbf{Multi-object tracking and data association.}
Data association is central in MOT~\cite{dendorfer2020mot20benchmarkmultiobject}. SORT~\cite{bewley2016simple} and IOU Tracker~\cite{bochinski2018high} use Hungarian matching. ByteTrack~\cite{zhang2022bytetrack}, OC-SORT~\cite{cao2023observation}, BoT-SORT~\cite{aharon2022bot}, and Deep OC-SORT~\cite{maggiolino2023deep} refine tracking in dense scenes. Recent \emph{Pattern Recognition} works include Chan et al.~\cite{chan2022online}, Zhu et al.~\cite{zhu2025visible}, and Tran et al.~\cite{Van_Ma_2024}. Stadler and Beyerer~\cite{stadler2022modelling} explicitly model crowd assignment ambiguity.

\textbf{Online and incremental algorithms.}
Online bipartite matching~\cite{karp1990optimal} maximizes cardinality without weight optimization. Theoretical online allocation extensions were explored by Mehta et al.~\cite{mehta2007adwords}.

\textbf{Crowd analysis and benchmarking.}
Crowded scene analysis is surveyed by Li et al.~\cite{li2015crowded}. Pedestrian group behavior is studied by Moussa\"id et al.~\cite{moussaid2010walking}, Zheng et al.~\cite{zheng2020study}, and Ge et al.~\cite{GeCR12}. Helbing and Moln\'ar~\cite{helbing1995social} introduced the social force model. Standard crowd benchmarks include MOT16~\cite{milan2016mot16}, CrowdHuman~\cite{shao2018crowdhuman}, and MOT20~\cite{dendorfer2020mot20benchmarkmultiobject}, evaluated using HOTA~\cite{luiten2021hota}.

\FloatBarrier

\section{Conclusion}
\label{sec:conclusion}

We presented an incremental optimal assignment algorithm tailored to crowd tracking scenarios. By chaining $N$ incremental steps starting from dual potentials optimal for subproblem $n-1$, each step requires a single augmenting path search. Combined with \textsc{SparseReorder} and cumulative-offset priority queues, our solver achieves 3.7--6.5$\times$ speedups over dense Hungarian baselines on crowd benchmarks up to $N=5000$. Future work includes decremental matching for object disappearances, re-identification costs, and GPU parallelization.

\section*{Acknowledgements}
The authors thank the anonymous reviewers for their constructive feedback.

\section*{Declaration of generative AI and AI-assisted technologies}
During manuscript preparation, the author used generative AI tools for spell checking, grammar correction, text improvement, and formatting. The author reviewed and edited all content and takes full responsibility for the published article.

\section*{Contributions of the Authors}
All work, including algorithmic development, coding, testing, and writing, was performed by the single author.


\end{document}